\documentclass[sigconf]{acmart}
\usepackage{todonotes}

\usepackage{fancyhdr}
\fancypagestyle{plain}{%
	\fancyhf{} 
	
}
\AtBeginDocument{%
  }

\setcopyright{acmlicensed}
\copyrightyear{2018}
\acmYear{2018}
\acmDOI{XXXXXXX.XXXXXXX}

 \usepackage{enumitem}
\usepackage{todonotes} %
\usepackage[most]{tcolorbox}
\usepackage{xspace}
\usepackage{makecell}

\newcommand{\approach}{\textbf{\textit{BIS}}\xspace}

\newcommand{\answer}[2]{
  \begin{tcolorbox}[enhanced, 
  breakable,
  left=3mm,right=3mm,
    colback=gray!10, colframe=gray!80, boxrule=0pt,
    borderline west={4pt}{0pt}{gray!90},
    ]
    \textbf{Answer for RQ#1:}
    #2
    \end{tcolorbox}
}

\begin{document}

\title{Importance Sampling is All You Need: Predict LLM's performance on new benchmark by reusing existing benchmark.
}
\author{Junjie Shi}
\email{CHUNKIT001@e.ntu.edu.sg}
\affiliation{%
  \institution{Nanyang Technological University}
  \country{Singapore}
}

\author{Wei Ma}
\email{weima@smu.edu.sg}
\authornote{Corresponding author}
\affiliation{%
  \institution{Singapore Management University}
  \country{Singapore}
}

\author{Shi Ying}
\email{yingshi@whu.edu.cn}
\affiliation{%
  \institution{Wuhan University}
  \city{Wuhan}
  \state{Hubei}
  \country{China}
}

\author{Lingxiao Jiang}
\email{lxjiang@smu.edu.sg}
\affiliation{%
  \institution{Singapore Management University}
  \country{Singapore}
}

\author{Yang Liu}
\email{yangliu@ntu.edu.sg}
\affiliation{%
  \institution{Nanyang Technological University}
  \country{Singapore}
  }

\author{Bo Du}
\email{dubo@whu.edu.cn}
\affiliation{%
  \institution{Wuhan University}
  \city{Wuhan}
  \state{Hubei}
  \country{China}
}

\renewcommand{\shortauthors}{J. Shi and W. Ma et al.}

\begin{abstract}
With the rapid advancement of large language models (LLMs), code generation has become a key benchmark for evaluating LLM capabilities. However, existing benchmarks face two major challenges: (1) the escalating cost of constructing high-quality test suites and reference solutions, and (2) the increasing risk of data contamination, which undermines the reliability of benchmark-based evaluations.

In this paper, we propose \approach{}, a prompt-centric evaluation framework that enables ground-truth-free prediction of LLM performance on code generation tasks. Rather than executing generated code, \approach{} estimates performance metrics by analyzing the prompt distribution alone. Built on importance sampling theory and implemented using Importance Weighted Autoencoders~(IWAE), our method reweights samples from existing annotated benchmarks to estimate performance on new, unseen benchmarks. To stabilize the estimation, we introduce weight truncation strategies and compute marginal expectations across the fitted distributions. \approach{} serves as a complementary tool that supports benchmark development and validation under constrained resources, offering actionable and quick feedback for prompt selection and contamination assessment.

We conduct extensive experiments involving 8,000 evaluation points across 4 CodeLlama models (7B–70B) and 9 diverse benchmarks. Our framework achieves an average absolute prediction error of 1.1\% for code correctness scores, with best- and worst-case errors of 0.3\% and 1.9\%, respectively. It also generalizes well to other metrics, attaining average absolute errors of 2.15\% for pass@1. These results demonstrate the reliability and broad applicability of \approach{}, which can significantly reduce the cost and effort of benchmarking LLMs in code-related tasks. %
 \end{abstract}



\maketitle
\pagestyle{plain}

\section{Introduction}
\label{sec:intro}
Large Language Models (LLMs) have demonstrated exceptional capabilities across various software engineering tasks, with code generation being a particularly notable example. As coding becomes one of the most critical benchmarks for assessing the performance of LLMs, it is essential to evaluate their coding abilities rigorously~\cite{chen2024survey}. To this end, numerous benchmarks have been developed to measure LLMs' coding proficiency, including SWE-Bench~\cite{du2025swe}, HumanEval~\cite{chen2021evaluating}, and BigCodeBench~\cite{zhuo2024bigcodebench}.
Despite their utility, current LLM evaluation benchmarks face two major challenges:

\textbf{High Benchmark Development Costs}: Constructing reliable benchmarks requires significant manual effort, especially for developing detailed test suites and reference solutions. For example, BigCodeBench~\cite{zhuo2024bigcodebench} required the manual creation of over 5,000 tests. Attempts to automate test suite generation have been inadequate, with current automated methods achieving only around 39.2\% accuracy~\cite{jain2024testgeneval}.

\textbf{Risk of Data Contamination}: Publicly available benchmarks and test suites can unintentionally become part of the LLM training datasets directly or indirectly, leading to artificially inflated performance over time and undermining benchmark reliability.

Motivated by these limitations, we introduce a novel evaluation paradigm (\approach, Prompt Importance Sampling) inspired by a previously overlooked insight: under fixed trained-well LLM parameters and evaluation criteria, prompt distributions inherently determine the LLM observed performance. By leveraging this relationship, we propose an evaluation framework that can estimate LLM performance by analyzing prompt distributions alone—without executing generated code or relying on costly ground-truth solutions.

To the best of our knowledge, this is the first work to predict LLM code generation performance without ground-truth execution. Our framework is grounded in importance sampling, a statistical technique widely used in off-policy reinforcement learning~\cite{fujimoto2019off,munos2016safe,uehara2022review}, which allows us to estimate expectations under a target distribution using samples from a different, known distribution.

In our setting, this means estimating the expected performance of an LLM on a new set of prompts by reweighting existing prompts from known benchmarks. We achieve this by modeling prompt distributions using Importance Weighted Autoencoders~(IWAE), which are well-suited for capturing complex, multimodal distributions~\cite{burda2015importance}. By learning rich latent representations of prompt distributions, IWAE enables effective approximation of the importance weights, making it possible to reuse prior benchmark data to predict performance on novel code generation tasks.

Our method complements ground-truth-based evaluation by alleviating two key challenges: reducing benchmark construction costs and mitigating data contamination risks.
First, to mitigate high benchmark development costs, \approach{} enables early-stage performance estimation by transferring knowledge from existing, fully annotated benchmarks. This allows benchmark designers to assess the value of candidate tasks before committing to expensive test suite construction.
Second, to help reduce the risk of data contamination, \approach{} provides an indirect signal of model familiarity with the prompt distribution. By analyzing discrepancies between expected and predicted performance patterns, it can help flag tasks that may overlap with training data, even in the absence of ground-truth execution.
This complementary role positions \approach{} as a valuable tool for guiding benchmark design and validation, especially when operating under limited resources or concerns about data leakage.

In theory, we analyze why \approach\ is effective and how it provides an unbiased estimation of LLM performance on the target benchmark. Furthermore, we discuss the practical feasibility of our approach. Additional details are provided in \autoref{sec:unbiased}.
We conducted extensive experiments involving approximately 8,000 data points across 4 LLMs (CodeLlama models ranging from 7B to 70B parameters) and 9 diverse benchmarks. Results demonstrate that our framework achieves an average prediction error as low as 0.9\% on code correctness metrics, outperforming baseline distribution modeling methods. Additionally, our framework accurately generalizes to other crucial code-related metrics, including maintainability and security scores, validating its broad applicability.

In summary, the contributions of this paper include:
\begin{enumerate}
    \item Introducing the theoretical formalization explicitly linking prompt distributions to LLM performance, providing a rigorous conceptual foundation for prompt-based evaluation methods.
    \item Proposing the first prompt-centric, ground-truth-free evaluation framework, \approach, tailored specifically for code-generation tasks, effectively addressing high benchmark development costs and data contamination risks.
    \item Innovatively integrating importance sampling with IWAE, empirically validated through comprehensive experiments demonstrating improved accuracy, interpretability, and robustness.
\end{enumerate}

The remainder of the paper is organized as follows: \autoref{sec:related} discusses related work. \autoref{sec:study} presents our methodology. \autoref{sec:evaluation} reports the experimental evaluation. \autoref{sec:discussion} discusses implications and limitations. \autoref{sec:conclusion} concludes.

\section{Related work}
\label{sec:related}
\subsection{{Code Benchmark}}
So far, benchmarking for LLMs in code generation now encompasses a wide range of scenarios, including smart contracts\cite{wijayakoon2025legal}, realistic environment settings\cite{du2025swe}, and large-scale codebases\cite{jimenez2023swe}. These benchmarks typically comprise at least two core components: (1) the code evaluation prompt, and (2) the test suite. Many benchmarks additionally incorporate reference solutions, problem sources, and other metadata. During evaluation, the LLM receives the code evaluation prompt as input and generates an output. This output undergoes formatting before being executed against the test suite. Consequently, a high-quality test suite constitutes a critical element of a valid and reliable code benchmark. The development of such test suites frequently necessitates manual curation by human experts, resulting in significant costs. Some alternative ground-truth-free methodologies has been proposed in other domains, such as benchmark-free scoring employing human or LLM-based reviewers. But they are unsuitable for code tasks given they may introduce extra bias and substantial cost. 

\subsection{Importance Sampling}
Importance sampling\cite{tokdar2010importance} is a Monte Carlo method\cite{james1980monte} designed to approximate expectations under a target distribution that is difficult to sample from directly, achieved by reweighting samples drawn from an accessible proposal distribution. Within this framework, two distinct distributions are involved. First, the target distribution $p(x)$ is intractable to sample from directly. Second, the proposal distribution $q(x)$ is from which samples can be readily generated. 
The objective of importance sampling is to estimate the expectation of a function $f(x)$ with respect to the target distribution $p(x)$. This is formally expressed as: $\mathbb{E}_{x \sim p}[f(x)] = \int f(x)p(x)dx $.
Importance Sampling estimates the expectation under a target distribution by weighting samples drawn from a proposal distribution. This process is formalized as follows, where $x_i$denotes a sample drawn from $q$: $\mathbb{E}_{x \sim p}[f(x)] \approx \frac{1}{N} \sum_{i=1}^{N} f(x_i) \cdot \frac{p(x_i)}{q(x_i)}, \quad x_i \sim q(x)$.

As shown in the equation above, the importance weight for sample $x_i$ is obtained by dividing the marginal probability of $x_i$ under the target distribution by its marginal probability under the proposal distribution. 
Within reinforcement learning, importance sampling is extensively utilized in scenarios characterized by data scarcity or challenging sampling conditions\cite{levine2020offline}. When direct sampling from the target environment proves difficult or costly, trajectories generated by a behavioral policy may be employed for policy evaluation. By adjusting these trajectories through importance sampling weights, the performance of the target policy can be estimated without altering the policy itself. This technique offers several distinct advantages:

\textit{Distributional Agnosticism}: Importance sampling does not rely on a pre-specified family of probability distributions. It requires no assumption that the data adheres to a specific parametric form, thereby accommodating a  broad spectrum of distribution types.

\textit{Mathematical Grounding and Interpretability}: By computing weight explicitly, importance sampling provides inherent interpretability and is unbiased estimation\cite{tokdar2010importance}.

\subsection{Variational Inference and Variational Autoencoder}
Variational inference(VI)~\cite{blei2017variational} is a technique for approximating complex posterior distributions given observed data. Computationally, VI achieves this approximation by maximizing the evidence lower bound (ELBO) on the marginal likelihood of the target distribution. Fundamentally, this approach minimizes the Kullback-Leibler (KL) divergence\cite{van2014renyi} between the approximating distribution and the true target distribution, as formalized below:
$ \log p(x) = \mathcal{L}(q) + D_{KL}(q(z; \theta) \| p(z|x))$,
where $ p $ denotes the target posterior distribution, while $ q $ represents the approximating distribution. The parameters of the $ q $ distribution are denoted by $ \theta $, and $ z $ refers to a sample drawn from the latent variable space. The observed data is represented by $ x $, and $ \mathcal{L}(q) $ stands for the ELBO, which serves as a variational objective function.

Variational Autoencoder (VAE)~\cite{pinheiro2021variational}  integrates variational inference with deep learning, commonly employed for unsupervised learning tasks. Its architecture comprises an encoder layer, which maps input data to a latent space; a sampling layer, responsible for drawing samples from this latent space; and a decoder layer, which reconstructs the input from the sampled latent representations. The VAE uses ELBO as reconstruction loss and uses KL divergence as a regularization term.

IWAE (Importance Weighted Autoencoder), a variant of VAE, achieves a tighter ELBO by performing multiple latent space samplings and weighting the importance of each sample. The ELBO constructed by IWAE is given below:
\begin{equation}
ELBO_{IWAE} = \mathbb{E}_{z_1,\ldots,z_K \sim q(z|x)} \left[ \log \left( \frac{1}{K} \sum_{k=1}^{K} \frac{p(x,z_k)}{q(z_k|x)} \right) \right]
\end{equation}.

Where $K$ is the number of samples, $z_k \sim q(z|x)$ are $K$ samples drawn from the approximate posterior distribution, $x$ is the observed data, and $q(z_k|x)$ is the variational approximation produced by IWAE . The term $p(x,z_k)$ represents the joint distribution of the observed data and the sampled latent variables. As $K$ increases, the ELBO  of IWAE approaches the log-likelihood of the target distribution more closely. When $K=1$, IWAE reduces to the standard VAE.

\section{Study Design}
\label{sec:study}
\subsection{Motivation}
Our motivation originates from a basic yet under-explored observation: the performance of LLMs in code generation tasks is highly sensitive to variations in prompt wording, structure, and contextual details. While it is widely recognized that prompts significantly influence the performance of code-generating LLMs, existing evaluation methodologies largely neglect leveraging prompt distributional characteristics. Instead, they heavily rely on executing generated code against manually crafted test suites—an approach inherently expensive and prone to data contamination issues if there is not quick feedback.

A natural yet unexplored insight from our formal analysis reveals that when the model and benchmark conditions are fixed, prompt distribution itself uniquely determines the expected model performance. While intuitive, this theoretical equivalence has been overlooked by prior works, which inherently assumed that explicit code execution and reference solutions are indispensable for accurate benchmarking.

Some of notations we use in this paper is shown in \autoref{tab:notation}.
Formally, consider an LLM as a probabilistic generative model $P_{LLM}(c|t)$. Given a fixed model and benchmark evaluation metrics $f_{b}$, the expected score for a given prompt distribution $P_{t}$ can be succinctly represented as:
\begin{equation}
\label{definition}
\mathbb{E}_{t \sim P_{t}}[score]=\int_{t}\int_{c} f_{b}(c) \cdot P_{LLM}(c|t) \cdot P_{t}(t) dcdt
\end{equation}

This equation reveals that, under fixed model parameters and evaluation setup, performance estimation fundamentally reduces to capturing the underlying prompt distribution $P_t$, which has not been previously exploited explicitly. Leveraging this insight, our work proposes an innovative, prompt-centric evaluation framework that eliminates the need for code execution and expensive ground-truth test suites.

\begin{table}[htbp]
\centering
\caption{Notation}
\label{tab:notation}
\resizebox{0.46\textwidth}{!}{
\begin{tabular}{@{}lc@{}}
\toprule
\textbf{Notation} & \textbf{Meaning} \\ 
\midrule
\texttt{t} & prompt \\
\texttt{c} & code \\
\texttt{b} & benchmark \\
$f_{b}$ & the evaluation function in benchmark \\
$P_{LLM}$ & the output distribution of LLM \\
$P_{t}$ & the output distribution of prompt \\
$P_{t}^{source}$ & the output distribution of source prompt set \\
$P_{t}^{target}$ & the output distribution of target prompt set \\
$E_{P_{t}}(score)$ & the expectation score of LLM on prompt distribution $P_{t}$\\
$\mathbf{t}_{\text{source}}^{(i)}$ & the i-th prompt from source prompt set \\
$\mathbb{IWAE}_{source}$ & The IWAE model trained with source prompt set \\
$\mathbb{IWAE}_{target}$ & The IWAE model trained with target prompt set \\
$\hat{\mathbf{y}}_{\text{target}}$& Prodiction score for target prompt set \\
\bottomrule
\end{tabular}
}
\end{table}

In practice, we use average values of all samples to estimate expected score of LLM on $P_{t}$.
 \begin{equation}
     \mathbb{E}_{t \sim P_{t}}{[score]} \approx \frac{1}{N}\sum_{i=1}^{N}f_{\text{b}}(\text{c}_i)
 \end{equation}
However, as noted previously, the test suite embodied by $f_{b}$ is often challenging to develop in practice. Consequently, for a given codebase, we may find that $f_{b}$ is difficult to obtain. Therefore, we return to the initial point. In \autoref{definition}, we note that although actual
implementations require testing suites for distinct code samples in practice, $f_{b}$  can theoretically be regarded as a deterministic and known function. On the other hand, we can assume LLM is also known. This holds because while the precise mechanism by which the LLM generates its probability distribution remains undisclosed, we can approximate this distribution through sampling. This implies that to obtain the expectation in Equation 4, we only need to compute the distribution corresponding to the prompts. Thus, we have successfully formalized the statement and proven our conjecture: by determining the distribution of prompts, we can derive the expected performance of the LLM on a given code benchmark.

We emphasize that while intuitive, this prompt-centric approach has remained unexplored in the literature. One possible reason for its neglect is the implicit assumption in prior works that explicit execution and verification through test suites are indispensable. However, as we demonstrate empirically, prompt distributions carry rich and sufficient information to reliably estimate model performance. Thus, our approach not only introduces significant efficiency improvements but also provides a novel evaluation methodology that addresses the growing challenges of benchmark cost and data contamination in code-generation tasks.

\subsection{Methodology}
Our approach builds directly upon the key observations established in the motivation: under a fixed LLM and evaluation metric, the expected model performance is entirely determined by the distribution of input prompts (Equation~\ref{definition}). Based on this, we aim to estimate the performance of an LLM on a new benchmark (the \emph{target prompt set, $P_t^{target}$}) without executing code, by reweighting existing performance data from an annotated \emph{source prompt set, $P_t^{source}$}.

To operationalize this idea, we adopt an importance sampling framework (\approach{}) as shown by \autoref{fig:framework}, which enables estimating expectations under a target distribution using samples drawn from a different, known source distribution. 

\approach{} takes both the prompt for a target code generation task and an LLM as inputs. It employs an \textbf{Embedding Module} to obtain the prompt embedding.
We model both the source and target prompt distributions using Importance-Weighted Autoencoders (IWAE)~\cite{burda2015importance}, a latent-variable model well-suited for approximating complex, multimodal distributions. The IWAE provides tractable estimates of the marginal likelihood of prompts under each distribution, from which importance weights can be derived.
The \textbf{IWAE Module} is responsible for modeling the distribution $P_{t}^{\text{target}}$ of target prompts by learning model $\mathbb{IWAE}_{target}$ from the feature vectors from \textbf{Embedding Module}. 

\begin{figure*}[]
  \centering
  \includegraphics[width=0.8\linewidth]{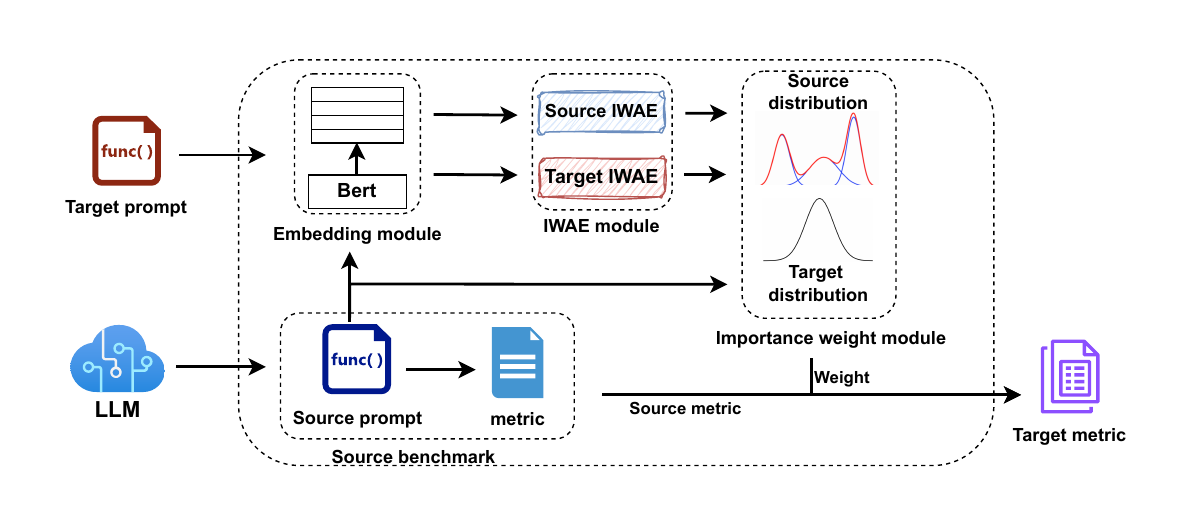}
  \vspace{-2em}
  \caption{Workflow of our framework, \approach{}.}
  \label{fig:framework}
\end{figure*}

Simultaneously, the other \textbf{IWAE Module} models the distribution $P_{t}^{\text{source}}$ of prompts by learning model $\mathbb{IWAE}_{source}$ from the source benchmark. Concurrently, the \textbf{pre-loaded source benchmark} within our framework is evaluated on the LLM, yielding test results $\mathcal{R}_{\text{source}} = \{r_1, r_2, ..., r_N\}$.

For each prompt sample $\mathbf{t}_{\text{source}}^{(i)}$ within the source benchmark, we infer an \textit{importance weight} $w_i$ using aforementioned two trained IWAE models. The importance weight $w_i$ can be obtained by the division of the marginal probability of $\mathbf{t}_{\text{source}}^{(i)}$ in $\mathbb{IWAE}_{target}$ and the marginal probability of $\mathbf{t}_{\text{source}}^{(i)}$ in $\mathbb{IWAE}_{target}$:
\begin{equation}
    w_{i} = \frac{\mathbb{IWAE}_{target}(\mathbf{t}_{\text{source}}^{(i)})}{\mathbb{IWAE}_{source}(\mathbf{t}_{\text{source}}^{(i)})}
\end{equation}

The final prediction score $\hat{\mathbf{y}}_{\text{target}}$ for the target code generation task prompt is then computed as the weighted output based on these importance weights:
\begin{equation}
\hat{\mathbf{y}}_{\text{target}} = \sum_{i=1}^{N} w_i \cdot f_{b}(P_{\text{LLM}}(\mathbf{t}_{\text{source}}^{(i)}))
\label{eq:weighted_prediction}
\end{equation}
where $P_{\text{LLM}}$ denotes the LLM's response function.

Specifically, our framework \approach{} comprises four modules:

\textbf{Embedding Module.} We employ a BERT model to extract high-dimensional embeddings from the target prompt as features. To ensure the extracted features are both expressive and concise, we utilize the embedding corresponding to the <CLS> token position within BERT as our feature representation.

\textbf{Source Benchmark Module.} This module contains a comprehensive code benchmark, incorporating a source prompt set from an alternative code task dataset along with their corresponding test suites and evaluation metrics. It also takes an LLM as input and outputs the LLM's performance metrics on the source prompts. 

\textbf{IWAE Module.} This module is responsible for fitting the posterior distribution over the features of the source prompts and target prompts, outputting corresponding IWAE models for the two prompt sets. We adopt the Importance Weighted Autoencoder (IWAE), a variant of  VAE, as our distribution fitting method. 

\textbf{Importance Weight Module.} This module takes as input two IWAE models and a source prompt set generated by the IWAE module. It computes an importance weight for each sample within the source prompt set. These weights are normalized and subsequently multiplied by the respective sample scores to yield the final prediction.

Within this module, the importance weight for each sample is calculated as the ratio of its marginal probability under the proposal distribution to its marginal probability under the target distribution, using the two IWAE models. To mitigate the potential instability arising from extreme weights—which can occur when high-probability regions of the proposal and target distributions are misaligned, leading to prediction dominance by a small number of samples—we employ weight clipping. This technique enhances the robustness of our predictions.

\subsection{Unbiased Estimation Analysis of \approach{}}
\label{sec:unbiased}
This section proves our framework is capable of providing unbiased estimation without implementing $f_{\text{b}}$ for prompts sampled from $P_{\text{t}}^{\text{target}}$. There is one key assumption for the following theorem.

\textbf{Assumption 1:}  The target distribution is absolutely continuous with respect to the source distribution. Specifically, for all prompts $t$ satisfying $P_{\text{t}}^{\text{target}}(t) > 0$, we have $P_{\text{t}}^{\text{source}}(t) > 0$.

\begin{theorem}
\label{lem:lemma1}
Given source prompt distribution $P_{\text{t}}^{\text{source}}$ and target prompt distribution $P_{\text{t}}^{\text{target}}$, under the above conditions, we can provide unbiased estimates for LLM's score on target prompt set $\mathbf{E}_{P^{\text{target}}_{\text{t}}}(f_{b} \cdot P_{\text{LLM}})$ using prompts sampled from $P_{\text{t}}^{\text{source}}$ instead of $P_{\text{t}}^{\text{target}}$.
\end{theorem}

\begin{proof}
We employ importance sampling to rewrite the target expectation in terms of the source distribution:
\begin{equation}
\label{proof:3.1}
    \begin{aligned}
        &\mathbf{E}_{P^{\text{target}}_{\text{t}}}(f_{b} \cdot P_{LLM} ) \\
        &= \int_{\text{t}} \int_{\text{c}} f_{\text{b}}(\text{c}) \cdot P_{\text{LLM}}(\text{c}|\text{t}) \, \cdot P^{\text{target}}_{\text{t}}(t)  \,d\text{c} d\text{t} \\
     &=\int_{\text{t}} \int_{\text{c}} f_{\text{b}}(\text{c}) \cdot P_{\text{LLM}}(\text{c}|\text{t}) \, \cdot P^{\text{source}}_{\text{t}}(t)   \cdot \frac{P^{\text{target}}_{\text{t}}(t)}{P^{\text{source}}_{\text{t}}(t)} \,d\text{c} d\text{t} \\
     &=\mathbf{E}_{P^{\text{source}}_{\text{t}}}(f_{b} \cdot P_{LLM} \cdot
     \frac{P^{\text{target}}_{\text{t}}(t)}{P^{\text{source}}_{\text{t}}(t)})
    \end{aligned}
\end{equation}

Assumption 1 ensures that the target dataset should have the statistical relationship to the source distribution, thereby keeping importance weights finite.
The second equality follows from this assumption, which ensures the density ratio is well-defined wherever $P^{\text{target}}_{\text{t}}(t) > 0$. The final equality reinterprets the integral as an expectation with respect to the source distribution, completing the importance sampling transformation.
\end{proof}

This result demonstrates that we only need prompts sampled from $P^{\text{source}}_{t}$, and there is no need to implement $f_{b}$ for prompts sampled from $P^{\text{target}}_{t}$. The unbiased estimator is given by:
$$\hat{\mu} = \frac{1}{n} \sum_{i=1}^{n} f_{b}(c^{(i)}) \cdot \frac{P^{\text{target}}_{\text{t}}(t^{(i)})}{P^{\text{source}}_{\text{t}}(t^{(i)})}$$
where $t^{(i)} \sim P^{\text{source}}_{\text{t}}$ and $c^{(i)} \sim P_{\text{LLM}}(\cdot|t^{(i)})$.

\subsubsection{\textbf{Practical Feasibility and Effect of Assumption~1.}}
While Assumption 1 requires certain overlap between source and target prompt distributions, this condition is difficult to fully meet in cross-domain or cross-lingual settings, potentially resulting in increased variance of importance weights. Nevertheless, \approach{} still offers the following unique advantages:

\begin{enumerate}[itemsep=2pt,leftmargin=*]
  \item \textbf{Cost-Effective Preliminary Screening.}  
        When evaluating multiple candidate benchmarks, \approach{} offers a coarse performance ranking at low computational cost (requiring only prompt embedding calculations). This enables researchers to prioritize limited resources toward the most promising directions, avoiding the high expense of blindly constructing test suites.
  \item \textbf{Quantitative Diagnosis of Distribution Shift.}  
        Through statistical characteristics of importance weights (e.g., weight variance, extreme value ratio), \approach{} quantifies the alignment between source and target distributions. This provides data-driven support for subsequent evaluation strategy selection: prompts exhibiting high weight variance necessitate either source dataset augmentation or hybrid evaluation approaches.
    \item \textbf{Significant Mitigation of Data Contamination Risks.}By eliminating reliance on complete test suites and reference answers, \approach{} inherently avoids the training data leakage issues plaguing traditional benchmarks. This holds substantial methodological significance in the era of large language models.
\end{enumerate}

When importance weights exhibit extreme distributions, it indicates severe violation of Assumption ~1. In such cases, we recommend: (1) augmenting the diversity of the source benchmark, (2) performing traditional evaluation on samples with anomalous weights, or (3) employing robust estimation via weight truncation.

\subsection{Quantitative Analysis: A formal analysis on error upper bound}
In this section, we conduct a theoretical analysis of our framework to obtain a quantitative understanding of its generalization ability. We formalize the problem setup, define key symbols, and derive the loss function under standard assumptions.

\paragraph{Problem Setup and Symbol Definitions}
We adopt the following assumptions and definitions:
\begin{enumerate}
    \item \textbf{VAE Distribution:} IWAE learns a distribution $\mathbb{IWAE}_{target}(\mathbf{x})$ to approximate true distribution of source prompt set $P_{t}^{source}$. Another IWAE model is used to approximate target distribution $\mathbb{IWAE}_{target}(\mathbf{x})$ from $P_{t}^{target}$, which is true distribution of target prompt set. 
    \item \textbf{Objective:} Estimate the expectation $\mu = \mathbb{E}_{\mathbf{x} \sim P_{t}^{source}}[f(\mathbf{x})]$, where $f(\cdot)$ is benchmark metric $\mu$ can be written as:
    \begin{equation}
    \label{equ:23}
    \mu = \frac{1}{n} \sum_{i=1}^{n} \frac{P_{t}^{target}(\mathbf{x}_i)}{P_{t}^{source}(\mathbf{x}_i)} f(\mathbf{x}_i).
    \end{equation}
    \item \textbf{Importance Sampling Estimator:} Given $n$ samples $\mathbf{x}_i \sim P_{t}^{source}$, the IS estimator is defined as:
    \begin{equation}
    \label{equ:22}
    \hat{\mu}_{\text{IS}} = \frac{1}{n} \sum_{i=1}^{n} \frac{\mathbb{IWAE}_{target}(\mathbf{x}_i)}{\mathbb{IWAE}_{source}(\mathbf{x}_i)} f(\mathbf{x}_i).
    \end{equation}
    \item \textbf{Prediction Error:} The absolute error (AE) of the estimator is:
    \begin{equation}
    \text{AE}(\hat{\mu}_{\text{IS}}) = \left|\hat{\mu}_{\text{IS}} - \mu\right|.
    \end{equation}
\end{enumerate}
Inspired by the concept of chi-square divergence, we employ a function 
$\varepsilon(x)$ to characterize the difference between the posterior distribution approximated by the IWAE and the actual distribution. Assuming sufficient proximity between the IWAE and the actual distribution with adequate training time and sample size, we subsequently derive an upper bound on the prediction error of our framework:
\begin{theorem}
\label{lem:lemma2}
 Let $\mathbb{IWAE}_{target} = P_{t}^{target}(1+\varepsilon_{target}(x))$, where $|\varepsilon_{target}(x)| \ll 1$ and  $\mathbb{IWAE}_{source} = P_{t}^{source}(1+\varepsilon_{source}(x))$, where $|\varepsilon_{source}(x)| \ll 1$.

 We have:  
 \begin{equation}
    \text{AE}(\hat{\mu}_{\text{IS}}) \leq \left|{\mu}( sup(\varepsilon_{source}(x)-\varepsilon_{target}(x))) \right|
    \end{equation}
\end{theorem}
Proof.  We have
\begin{align}
\text{AE}(\hat{\mu}_{\text{IS}})  &= (\frac{1}{n} \sum_{i=1}^{n}( \frac{\mathbb{IWAE}_{target}(\mathbf{x}_i)}{\mathbb{IWAE}_{source}(\mathbf{x}_i)}  -   \frac{P^{target}_{pr}(\mathbf{x}_i)}{P^{source}_{t}(\mathbf{x}_i)} )f(\mathbf{x}_i))^{2*\frac{1}{2}} \\
     &=  (\frac{1}{n} \sum_{i=1}^{n} \frac{P^{source}_{t}(\mathbf{x}_i){}}{P^{target}_{t}(\mathbf{x}_i)}f(\mathbf{x}_{i})(1 - \frac{1+\varepsilon_{target}(\mathbf{x}_i))}{1+\varepsilon_{source}(\mathbf{x}_i))}))^{2*\frac{1}{2}}\\
     &\approx (\frac{1}{n} \sum_{i=1}^{n} \frac{P^{target}_{t}(\mathbf{x}_i)f(\mathbf{x}_i)}{P^{source}_{t}(\mathbf{x}_i)}( \varepsilon_{source}(\mathbf{x}_i)-\varepsilon_{target}(\mathbf{x}_i)))^{2*\frac{1}{2}} \\
     &\leq \left| \mu(sup(\varepsilon_{source}(\mathbf{x}_i)-\varepsilon_{target}(\mathbf{x}_i))\right| 
\end{align}
\qed

Since $\mu$ depends solely on the actual prompt distribution and the dataset, we consider it a constant. Consequently, our results indicate that when the learning of the Importance-Weighted Autoencoder (IWAE) sufficiently approximates the target distribution, \textit{the disparity in the fitting quality between the two IWAEs} dominates the upper bound of the mean squared error for importance sampling.

However, this conclusion holds under the prerequisite that \textit{both} IWAEs achieve sufficiently close approximations to their respective target distributions. Under such conditions, since both IWAEs exhibit high fitting quality, the disparity between them remains minimal. Nevertheless, individual \textit{out-of-distribution} samples may significantly elevate the loss upper bound. This occurs because the two IWAEs could produce markedly divergent responses to such samples. Therefore, we recommend maintaining diversity within the prompt sets in practical applications.
\section{Evaluation}
\label{sec:evaluation}
In this section, we conduct thorough experiments to validate the framework proposed in our study. We first describe our \textbf{experimental setup}, followed by the experimental design and results addressing the following five research questions:

\begin{enumerate}
    \item[\textbf{RQ1:}] Can our framework effectively predict model correctness metrics?
    
    \item[\textbf{RQ2:}] Can the predictive accuracy of our framework be generalized to other code-related metrics?
    
    \item[\textbf{RQ3:}] Do alternative methods achieve superior performance compared to our framework?
    
    \item[\textbf{RQ4:}] How do other factors (e.g., embedding dimensionality, prompt set size) influence the performance of our framework?
\end{enumerate}

\subsection{Setup}
To obtain high-quality and sufficient data for our experiments, we collected approximately 8,528 data points encompassing \textbf{4 models} and \textbf{9 benchmarks}. 
Considering the potential impact of data contamination on our experiments, we employed open-source models released before the publication of these benchmarks, deliberately excluding all closed-source models. Furthermore, to mitigate the risk that low scores of the chosen open-source models on these two code benchmarks would render most problems inconsequential to the final results, we utilized the CodeLlama\cite{roziere2023code} family for our experiments. It is worth noting that although models within this family share a common lineage, architectural differences exist among them, including variations in the number of layers, hidden dimensions, attention mechanisms, and training data.
The evaluation was conducted on the \texttt{CodeLlama} family of models, specifically including:
\texttt{CodeLlama-7B}, \texttt{CodeLlama-13B}, \texttt{CodeLlama-34B}, and \texttt{CodeLlama-70B}.

As shown in \autoref{tab:benchmark_stats}, the benchmark suite comprises:
BigCodeBench, HumanEval, EvoEval\cite{xia2024top}, which integrates 7 sub-benchmarks.

\begin{table}[]
\centering
\caption{Number of data points per benchmark}
\label{tab:benchmark_stats}
\scalebox{0.8}{
\begin{tabular}{@{}lc@{}}
\toprule
\textbf{Benchmark} & \textbf{Datapoints} \\ 
\midrule
\texttt{BigCodeBench} & 1, 140 \\
\texttt{HumanEval} & 164 \\
\texttt{EvoEval\_difficult} & 100 \\
\texttt{EvoEval\_creative} & 100 \\
\texttt{EvoEval\_subtle} & 100 \\
\texttt{EvoEval\_combine} & 100 \\
\texttt{EvoEval\_tool\_use} & 100 \\
\texttt{EvoEval\_verbose} & 164 \\
\texttt{EvoEval\_concise} & 164 \\
\midrule
\textbf{Total} & \textbf{2, 132} \\
\bottomrule
\end{tabular}}
\end{table}

To mitigate class imbalance across datasets and better approximate real-world application scenarios, we merged HumanEval with EvoEval to form the \textbf{Evo} dataset comprising 992 samples. Concurrently, BigCodeBench was maintained as a standalone \textbf{bigcode} dataset containing 1,140 samples. In subsequent experiments, for each large language model  evaluated, we implemented a reciprocal cross-prediction framework between these datasets: 
\begin{itemize}
    \item Performance on \emph{Evo} predicts performance on \emph{bigcode}
    \item Performance on \emph{bigcode} predicts performance on \emph{Evo}
\end{itemize}
This bidirectional evaluation protocol was systematically applied across all models. For multiple indicators spanning diverse value ranges, we employ min-max normalization to standardize all measurements within the unit interval $[0,1]$. After normalization, the average score of the LLM on the target prompt set is regarded as our prediction target, while the prediction error is calculated by subtracting this prediction target from our weighted score. To collect the data, we rented 8 server with L20 GPUs with the bill 280\$.

\subsection{RQ1: Predict Model Correctness
Metrics}

To represent correctness of code generated by LLM, we employ the \textbf{CodeBLEU score}~\cite{ren2020codebleumethodautomaticevaluation}, which integrates multiple similarity measures between the LLM-generated code and reference solutions. 
We conduct experiments to preliminarily validate the performance of our framework in predicting code correctness metrics. 
As shown in Table~\ref{tab:results}, our results demonstrate that when training on the \textit{BigCode} dataset to predict the \textit{Evo} dataset, 
the framework achieves an average absolute error rate of \textbf{1.05\%}.Conversely, when training on the \textit{Evo} dataset to predict the \textit{BigCode} dataset, it achieves an average absolute error rate of \textbf{1.15\%} .
We also computed the pass@1 metric. Due to computational constraints, this metric was calculated exclusively on CodeLlama-7B. For all benchmarks used in our study, we executed CodeLlama-7B ten times to determine the pass@1 score. The result is shown in \autoref{tab:passat1}.
These findings substantiate the robust predictive performance of our approach.

\begin{table}[]
\centering
\caption{Error rates of CodeBLEU score across datasets and language models}
\label{tab:results}
\scalebox{0.8}{
\begin{tabular}{llr}
\toprule
\textbf{Source dataset} & \textbf{LLM} & \textbf{Error}  \\
\midrule
BigCode                 & codellama-70b &  0.011 \\
                       & codellama-34b & 0.007 \\
                       & codellama-13b & 0.003 \\
                       & codellama-7b  &  -0.011 \\

Avg absolute Error& &\textbf{0.008}\\
\midrule
Evo                     & codellama-70b &  -0.016 \\
                       & codellama-34b &  -0.005 \\
                       & codellama-13b &  -0.015 \\
                       & codellama-7b  & -0.019 \\
Avg absolute Error& &\textbf{0.014}\\
\bottomrule
\end{tabular}}
\end{table}

\begin{table}[]
\centering
\caption{Error rates of pass@1 score on CodeLlama-7b}
\label{tab:passat1}
\scalebox{0.8}{
\begin{tabular}{llr}
\toprule
\textbf{Source dataset} & \textbf{BigCode} & \textbf{Evo}  \\
\midrule
\textbf{Error}  & -0.022 & 0.021\\

\bottomrule
\end{tabular}}
\end{table}

\answer{1}{Our framework can effectively predict LLM's performance on CodeBLEU and pass@1 metric with low absolute error.}

\subsection{RQ2: Evaluation On Other Metric}
In addition to accuracy metrics, recognizing the need for a more comprehensive evaluation of code generated by LLMs, several recent studies have begun incorporating alternative metrics that capture broader aspects of code quality—such as maintainability index~\cite{welker2001software} and security-related scores. In this work, we aim to extend the scope of importance sampling beyond simple accuracy measures to encompass these additional dimensions. We categorize the target metrics into two groups. The first group consists of \textbf{semantic-level} metrics, including cyclomatic complexity~\cite{ebert2016cyclomatic}(CC), and security scores (SS), which are derived from the semantic properties of the generated code. We consider these metrics to be particularly critical in evaluating the practical utility and robustness of LLM-generated code.

The metrics are formally defined as follows:
1). $\mathrm{CC} = E - N + 2P$,
where $E$ is edges, $N$ is nodes in control flow graph, and $P$ is connected components.
2). $\mathrm{SS} = \max(100 - {\sum_{i=1}^{n}w_i \cdot c_i},0)$,
where $v_i$ represents vulnerability severity and $c_i$ its confidence. In our framework, we choose bandit~\cite{kapustin2023static} as out static code auditory toolkit. The confidence and severity levels along with their corresponding weights are shown in \autoref{tab:weight}.

\begin{table}[ht]
\centering
\caption{Confidence level, severity and their corresponding weights in our framework}
\label{tab:weight}
\begin{tabular}{ll|ll}
\toprule
\textbf{Confidence level} & \textbf{weight} & \textbf{Severity} & \textbf{weight} \\
\midrule
High &  1 & High & 50    \\
Medium & 0.6  &Medium  & 30  \\
  Low  & 0.2 & Low & 10 \\

\bottomrule
\end{tabular}
\end{table}

These semantic metrics provide deeper insights into code quality compared to syntactic measures, evaluating complexity via $\mathrm{CC}$, and vulnerability risks through $\mathrm{SS}$. Our framework prioritizes these indicators as they directly reflect long-term software health and operational reliability.

Another category is \textbf{code-level} metrics, inspired by Halstead complexity~\cite{hariprasad2017software}. We classify those metrics that are directly related to the characters of the source code itself into this level. Specifically, they include program length, volume, effort, and time. These metrics originate from the basic elements of source code and quantify the size, complexity, and development cost of the program. Although they are not as closely related to runtime behavior as semantic-level metrics, they still play an important role in evaluating structural characteristics and coding conventions, especially suitable for preliminary analysis and comparison of code quality.

\begin{table}[hbt]
  \caption{Evaluation error on other metrics.}
  \label{tab:commands}
  \scalebox{0.6}{
  \large
  \begin{tabular}{c|c|cccc|cccc}
    \toprule
     &Source  dataset& \multicolumn{4}{c}{Bigcode}& \multicolumn{4}{c}{Evo}\\
    \cmidrule{3-10}
    Metric level&	Codellama- &7b&	13b&	34b	& 70b&	7b&	13b& 	34b&	70b\\
    \midrule
    \textbf{Semantic}	
    &	SS&	-0.039 & -0.037 & -0.040 & -0.040&	0.043 & 0.040 & 0.043 & 0.041 \\
    & CC&	-0.046 & -0.045 & -0.028 & -0.028 &	0.040 & 0.045 & 0.026 & 0.025 \\
    &Codebleu&	0.011 & 0.007 & 0.003 & 0.011	&-0.019 & -0.016 & -0.007 & -0.015 \\
    \midrule
    \textbf{Code}&	Length &	-0.046 & -0.045 & -0.048 & -0.031&	0.079 & 0.076 & 0.090 & 0.107 \\
    &Volume	&-0.035 & -0.037 & -0.037 & -0.022&	0.050 & 0.054 & 0.060 & 0.080 \\
    &Effort&	-0.022 & -0.022 & -0.023 & -0.018&	0.029 & 0.030 & 0.028 & 0.050\\
    &Time&	-0.021 & -0.022 & -0.023 & -0.018&	0.026 & 0.028 & 0.032 & 0.045 \\
    \bottomrule
  \end{tabular}}
\end{table}

Halstead complexity is a software measurement method proposed by Maurice Halstead, which is based on the number of operators and operands in a program. It measures the complexity and development workload of a program from the character level of the code. In our work, we classify a set of indicators inspired by Halstead complexity as code-level metrics, which mainly include the following four core indicators:

\begin{itemize}
    \item Length (Program Length, $L$): This refers to the total number of all operators and operands in the program. The calculation formula is: $L = n_1 + n_2$,
    where $n_1$ represents the number of operators, and $n_2$ represents the number of operands.

    \item Volume (Program Volume, $V$): This indicates the total information content of the program, measured in "bits," reflecting the minimum mental effort required to understand the program. Its formula is:$V = L \times \log_2(n_1 + n_2)$.

    \item Effort (Effort, $E$): This is used to estimate the degree of effort required by programmers to write the program, with units that can be understood as the number of basic operations. Its calculation method is: $E = \left(\frac{n_1}{2} + n_2\right) \times V$.

    \item Time (Time, $T$): This is an estimate of the development time for the program, usually measured in seconds. The calculation formula is:$T = \frac{E}{18}$.
\end{itemize}
Our evaluation result on above metric is shown in \autoref{tab:commands}. The experimental results demonstrate that our framework excels at predicting performance on semantic-level metrics compared to code-level metrics. The max absolute errors (MAEs) for the predictions of security score and cyclomatic complexity were 4.3\%, and 4.6\%, respectively. In contrast, the performance on code-level metrics was relatively weaker, with MAEs of 10.7\%, 8\%, 5\%, and 4.5\% for the four indicators of code length, volume, effort, and time, respectively. We believe this discrepancy may be attributed to the direct influence of randomness at the token-level output of the LLM on code-level metrics. In contrast, semantic-level metrics encode higher-level information, making them more stable and reliable in terms of prediction performance.

\answer{2}{The predictive capability of our framework can be extended to various metrics; however, given the same dataset scale, our framework performs better on semantic-level metrics.}

\subsection{RQ3: Comparsion With Baselines}
In this section, we compare our framework with alternative approaches and demonstrate that our method achieves state-of-the-art performance. We first investigate whether different fitting distributions can improve performance under the importance sampling framework. The baseline methods employed for comparison include: Gaussian Mixture Models\cite{reynolds2015gaussian} (GMMs), Restricted Boltzmann Machines\cite{fischer2012introduction} (RBMs), Conditional Maximum Entropy Models\cite{mcdonald2009efficient}, and Variational Autoencoders (VAEs).

For the GMM implementation, we evaluate two distinct configurations: one with 8 components and another with 80 components. Regarding the RBMs and Conditional Maximum Entropy Models, considering the learning challenges in high-dimensional sparse embedding spaces, we utilize both the original 768-dimensional inputs and their reduced 128-dimensional counterparts.Our test results is shown in \autoref{tab:method1}. We also consider the average value of  absolute value(avg of abs) of error for each solution.

From the experimental results, we observe that our approach achieves the lowest prediction error among all compared schemes. Furthermore, it can be noted that the performance of both Restricted Boltzmann Machines and Conditional Maximum Entropy models remains virtually unchanged after applying PCA dimensionality reduction. This finding suggests that these two methods are capable of effectively learning from sparse samples in high-dimensional spaces for our specific task.

Next, we compare our framework with non-importance sampling approaches. Since our objective only requires obtaining importance weights for each sample from the source prompt set, we formulate this as a regression problem and employ commonly used regression models for comparison.

The baseline models we employ can be categorized into two groups: (1) machine learning models that provide interpretability, such as Random Forest Regression\cite{segal2004machine}(RSR), Linear Regression\cite{su2012linear}(LR), Decision Tree Regression\cite{xu2005decision}(DTR), and Ridge Regression\cite{mcdonald2009ridge}(RR); and (2) deep learning-based models, for which we adopt Multilayer Perceptron\cite{popescu2009multilayer} (MLP) and Recurrent Neural Network\cite{sherstinsky2020fundamentals} (RNN) as our baselines. For the RNN implementation, instead of using only the embedding corresponding to the BERT-generated \texttt{<cls>} token as input, we sequentially feed all token embeddings generated by the model into the RNN. For the MLP, we directly train the model on the source prompt set and subsequently employ the trained MLP to predict the performance of samples in the target prompt set.

Our experimental results are presented in \autoref{tab:method2}. As can be observed, the interpretable machine learning-based approaches significantly outperform the non-interpretable deep learning-based methods. This performance gap may be attributed to the fact that MLP and RNN architectures are prone to either overfitting or underfitting phenomena. Furthermore, our proposed method still achieves the best performance among all compared solutions.
\begin{table}[]
  \centering
  \caption{Evaulation error  on other method under importance sampling framework.}
  \label{tab:method1}
  \resizebox{0.48\textwidth}{!}{%
  \Large
    \begin{tabular}{c|cccc|cccc|c}
      \toprule
      Source & \multicolumn{4}{c}{Bigcode} & \multicolumn{4}{c}{Evo}  \\
      \cmidrule{2-10}
      LLM & 7b & 13b & 34b & 70b & 7b & 13b & 34b & 70b  & avg of abs\\
      \midrule
      GMM(80)   & -0.265 & -0.256 & -0.221 & -0.302 & -0.268 & -0.303 & -0.333 & -0.374 & 0.290  \\
      GMM(8)    & 0.426 & -0.350 & -0.373 & -0.411 & -0.400 & -0.403 & -0.405 & -0.511 &0.328 \\
      RBM(768)  & -0.002 & -0.019 & -0.032 & 0.012  & 0.008  & 0.025  & 0.038  & -0.007 & 0.018 \\
      RBM(128)  & -0.002 & -0.019 & -0.032 & 0.012  & 0.008  & 0.025  & 0.038  & -0.007 & 0.018 \\
      MaxEnt(768)& -0.005 & -0.022 & -0.034& 0.009  & 0.005  & 0.0224 & 0.034 & -0.009 & 0.017 \\
      MaxEnt(128)& -0.005 & -0.022 & -0.034 & 0.009  & 0.005  & 0.022 & 0.034 & -0.009 & 0.017 \\
      VAE       & -0.009& -0.026& -0.020 &  0.006  & -0.032 & -0.003&  0.0028 & -0.024 &0.015 \\
      \textbf{Ours}   &0.011 & 0.007 & 0.003 & 0.011	&-0.019 & -0.016 & -0.007 & -0.015 &\textbf{0.011} \\
      \bottomrule
    \end{tabular}
  }
\end{table}

\begin{table}[]
  \caption{Evaulation error  on other method that are not under importance sampling framework. Notation: Ex=Explainable, Unex=Unexplainable}
  \label{tab:method2}
  \scalebox{0.6}{
  \begin{tabular}{c|c|cccc|cccc|c}
    \toprule
    &Source  dataset& \multicolumn{4}{c}{Bigcode}& \multicolumn{4}{c}{Evo}\\
    \cmidrule{3-11}
    &	Codellama- &7b&	13b&	34b	& 70b&	7b&	13b& 	34b&	70b&avg of abs\\
    \midrule
   \textbf{Ex} 	& RSR &	-0.033&	-0.045&	-0.060&	0.015&	0.019&	0.019&	0.034&	-0.029&0.031  \\
    &	LR& -0.005&	-0.022&	-0.034&	\textbf{0.009}&	\textbf{0.005}&	0.022&	0.034&	-0.009&0.017 \\
    & DTR&-0.032&	-0.054&	-0.065&	-0.026&	0.019&	0.017&	0.044&	-0.018&0.035 \\
    &RR&\textbf{-0.005}&	-0.022	&-0.034&	0.009&	0.005&	0.022&	0.034&	\textbf{-0.009}&0.017\\
    & \textbf{Ours}      & 0.011 & \textbf{0.007} & \textbf{0.003} & 0.011	&-0.019 & \textbf{-0.016} & \textbf{-0.007} & -0.015 & \textbf{0.011} \\
    \midrule
    \textbf{Unex}&	RNN &0.285&	0.278	&0.317&	0.278&	0.331&	0.318&	0.324&	0.314& 0.298\\
    & MLP	&-0.108	&-0.133&	-0.138&	-0.132	&0.061&	0.096	&0.093	&0.060& 0.105 \\
    \bottomrule
  \end{tabular}}
\end{table}
\answer{3}{Our solution outperforms all baseline methods in terms of prediction accuracy. We also observe that machine learning and statistical learning-based approaches perform better than deep learning-based methods, which may indicate that there is still room for development in applying deep learning models to this task.}

\subsection{RQ4: Ablation Study}
In this section, we investigate whether other factors in our framework could potentially affect its performance. Specifically, we focus on examining four key aspects:(1) Feature dimensionality. (2) The size of the prompt set. (3) The number of importance-weighted autoencoder samples. (4) The percentage ratios of truncated weight.
\begin{table}[]
\centering
\caption{Evaluation error of our framework with PCA.}
\label{tab:rd}
\scalebox{0.7}{
\Large
\begin{tabular}{ccccccc}
\toprule
\textbf{Source} & \textbf{Dimension} & \textbf{576} & \textbf{384}& \textbf{192}  \\
\midrule
BigCode    & codellama-70b&0.013&	0.017&0.016 \\
       & codellama-34b & 0.006&0.093&0.011	\\
    & codellama-13b & 0.009&0.104& 0.019 \\
    & codellama-7b  & 0.019& -0.021&  0.022  \\
\midrule
Evo  & codellama-70b &-0.017 &-0.039&  -0.016 \\
    & codellama-34b  &-0.007 & -0.016&-0.007\\
    & codellama-13b & -0.014&0.020&-0.015 \\
    & codellama-7b  & -0.021&-0.003& -0.022	\\
\bottomrule
\end{tabular}
}
\end{table}

\paragraph{The feature dimensionality.}
As the embedding dimensionality continues to increase, the sparse distribution of samples in high-dimensional space may lead to the curse of dimensionality. However, dimensionality reduction of samples could potentially result in information loss within the embeddings. We employed Principal Component Analysis\cite{abdi2010principal} (PCA) for dimensionality reduction. We subsequently examined the performance of our method in different dimensionality.Experimental result in shown in \autoref{tab:rd}.

The experimental results demonstrate that the absolute error value exhibits a generally increasing trend with decreasing dimensionality. While this relationship is not strictly monotonic, a positive correlation can be observed between the dimension reduction level and the magnitude of the final prediction error. It is also worth noting that although linear layers are a common dimensionality reduction approach, employing linear layers for dimensionality reduction in our framework leads to a significant increase in error, as shown in \autoref{tab:rdln}.

\begin{table}[]
\centering
\caption{Evaluation error of our framework with linear layer.}
\label{tab:rdln}
\scalebox{0.8}{
\begin{tabular}{ccccccc}
\toprule
\textbf{Source} & \textbf{Dimension} &\textbf{576} & \textbf{384}\\
\midrule
BigCode    & codellama-70b & -0.410&	-0.410\\
       & codellama-34b &-0.264	&-0.265\\
    & codellama-13b & -0.349&	-0.350 \\
    & codellama-7b  &  -0.290&	-0.291\\
\midrule
Evo  & codellama-70b &  -0.391&	-0.392 \\
    & codellama-34b &  -0.275&	-0.276\\
    & codellama-13b &  -0.319&	-0.320\\
    & codellama-7b  & -0.291	&-0.292\\
\bottomrule
\end{tabular}}
\end{table}

\paragraph{The number of importance-weighted autoencoder samples.}
In IWAE, appropriately increasing the number of samples helps IWAE better approximate the target distribution. However, excessively increasing the sample size may lead to large variance, thereby affecting test stability. Therefore, adjusting the number of samples in IWAE can help strike a balance between stability and accuracy. We experimentally investigate the relationship between our framework's performance and the number of samples, as shown in the \autoref{tab:sample}. It can be observed that when the sampling number equals 1, cases may occur where the absolute error exceeds 4\%. On the other hand, with sampling numbers more than 25, instances of absolute errors surpassing 3.5\% were detected. This empirical evidence substantiates our prior hypothesis that both excessively large and insufficiently small sampling numbers may compromise the stability of prediction outcomes.
\begin{table}[htbp]
  \centering
  \caption{Evaluation error on our method with different number of samples in IWAE.}
  \label{tab:sample}
  \scalebox{0.6}{%
  \Large
    \begin{tabular}{c|cccc|cccc}
      \toprule
      Source & \multicolumn{4}{c}{Bigcode} & \multicolumn{4}{c}{Evo} \\
      \cmidrule{2-9}
      Codellama & 7b & 13b & 34b & 70b & 7b & 13b & 34b & 70b \\
      \midrule
      100 & -0.005 & -0.028 & -0.031 & 0.017 & 0.006 & 0.025 & 0.031 & -0.012 \\
50 & -0.026 & -0.025 & -0.022 & 0.014 & \textbf{0.002} & -0.026 & 0.035 & 0.009 \\
25 & \textbf{0.004} & -0.022 & -0.039 & 0.014 & -0.008 & 0.011 & 0.028 & -0.014 \\
\textbf{10} & 0.011 &\textbf{0.007} & \textbf{0.003} & \textbf{0.011}	&-0.019 & -0.016 & -0.007 & -0.015 \\ 5 &0.004&	-0.011&	0.009&	0.017&	-0.016&	\textbf{0.010}&	\textbf{0.005}&	-0.014  \\
1 & -0.003 & -0.041 & -0.018 & -0.050 & -0.010 & 0.009 & 0.029 & -0.038 \\
      \bottomrule
    \end{tabular}
  }
\end{table}

\begin{table}[]
  \centering
  \caption{Evaluation error on our method with different percentile of truncated weight in IWAE.}
  \label{tab:truncate}
  \scalebox{0.7}{%
    \begin{tabular}{c|cccc|cccc}
      \toprule
      Source & \multicolumn{4}{c}{Bigcode} & \multicolumn{4}{c}{Evo} \\
      \cmidrule{2-9}
      Codellama & 7b & 13b & 34b & 70b & 7b & 13b & 34b & 70b \\
      \midrule
     1  & 0.037 & 0.019 & -0.013 & 0.049 & -0.024 & \textbf{0.011} & 0.010& \textbf{-0.010} \\
       0.95 & -0.021 & -0.022 & -0.023 & -0.018 &0.026 & 0.028 & 0.032 & 0.015 \\
      \textbf{0.9} & \textbf{0.011} &\textbf{0.007} & \textbf{0.003} & \textbf{0.011}	&\textbf{-0.019} & -0.016 & \textbf{-0.007} & -0.015 \\
      0.85 & -0.022 & -0.022 & -0.023 & -0.017 & -0.028 & 0.028 & 0.029 & 0.052 \\
      0.8 & -0.021 & -0.022 & -0.023 & -0.018 & 0.025 & 0.028 & 0.030 & 0.053 \\
     
      \bottomrule
    \end{tabular}
  }
\end{table}

\begin{table*}[htbp]
\centering
\caption{Bigcode $\leftrightarrow$ EVO Score Pairs.}
\vspace{-1em}
\label{tab:bigcode_evo_scores_70b}
\scalebox{0.7}{
\Large
\begin{tabular}{c|c|ccccc}
\toprule
CodeLlama & Setting & Full & 700 & 500 & 300 & 100 \\
\cmidrule{1-7}
70b &  Full    & -0.011 / -0.014 & 0.028 / 0.016 & 0.005 / 0.0022 & -0.017 / -0.012 & 0.015 / -0.028 \\
34b &     & 0.003 / -0.007 & -0.008 / 0.054 & -0.061 / 0.032 & -0.054/0.037 & -0.029/0.019 \\
13b &     & -0.011 / -0.010 & -0.012 / 0.053 & 0.030 / 0.020 & -0.058 / 0.019 & -0.040 / 0.009 \\
7b &     & 0.011 / -0.019	& 0.015 / 0.031	& -0.004 / 0.007 &	 -0.033 / -0.001	& -0.019 / -0.006 \\
\cmidrule{1-7}
70b & 700     & -0.024 / -0.042 & -0.006 / -0.013 & 0.003 / 0.006 & 0.034 / -0.011 & 0.029 / -0.006 \\
34b &     & -0.062 / 0.042&	-0.034 / 0.028	&-0.054 / 0.062	&0.004 / 0.049&	-0.054 / 0.051 \\
13b &     & -0.035 / 0.014&	-0.016 / 0.028&	-0.042 / 0.062&	0.003 / 0.042&	-0.003 / 0.058 \\
7b &     & -0.038 / -0.005&	-0.013 / 0.006&	-0.023 / 0.022&	0.021 / 0.013&	0.038 / 0.025 \\
\cmidrule{1-7}
70b & 500     & -0.028 / -0.009 & -0.023 / -0.054 & 0.001 / -0.026 & -0.006 / -0.039 & -0.025 / -0.002 \\
34b &     & -0.042 / 0.078&	-0.060 / -0.026&	-0.038 / 0.034	&-0.018 / 0.021	&-0.055 / 0.049 \\
13b &     & 0.014 / 0.018	&-0.045 / -0.009&	-0.032 / 0.034&	-0.039 / 0.038&	-0.024 / 0.060 \\
7b &     & -0.005 / 0.018&	-0.029 / -0.016&	-0.017 / 0.003&	0.002 / -0.010&	-0.024 / -0.031 \\
\cmidrule{1-7}
70b & 300     & -0.056 / -0.054 & -0.007 / -0.033 & 0.012 / -0.017 & -0.033 / -0.039 & -0.050 / -0.069 \\
34b &     & -0.080 / 0.016&	-0.031 / 0.024&	-0.034 / 0.047&	-0.083 / 0.014&	-0.064 / 0.011 \\
13b &     & -0.030/-0.030&	-0.037/-0.001&	0.040/-0.005&	-0.019 / -0.056&	-0.012 / -0.035 \\
7b &     & -0.006 / -0.031 & -0.012 / 0.053 & 0.030 / 0.020 & -0.058 / 0.019 & -0.040 / 0.009 \\
\cmidrule{1-7}
70b & 100     & -0.094 / -0.009 & 0.033 / 0.049 & -0.107 / -0.038 & -0.104 / -0.030 & -0.062 / -0.079 \\
34b &     & -0.118 / 0.045&	-0.090 / 0.003&	-0.148 / -0.011	&-0.026 / 0.037&	-0.173 / 0.058\\
13b &     & -0.098 / 0.040&	0.017 / 0.030&	-0.07 / -0.005&	-0.071 / 0.033&	0.100 / 0.092 \\
7b &     & -0.086 / -0.020&	-0.090 / 0.029&	-0.045 / -0.034&	-0.096 / -0.071&	0.147 / -0.090 \\
\bottomrule
\end{tabular}}
\end{table*} 

\paragraph{The percentage ratios of truncated weight.}
During importance sampling, when the proposal distribution $q(x)$ exhibits limited overlap with the high-probability regions of the target distribution $p(x)$, extreme importance weights may occur. This can lead to a situation where a minority of samples dominate the prediction, potentially compromising the stability of the estimation. To enhance the robustness of our predictions, we employ a percentile-based weight truncation strategy.

Specifically, we truncate all weights exceeding the $x$-th percentile to the value of the $x$-th percentile weight, followed by a normalization step to prevent underestimation. Let $w_{(1)}, w_{(2)}, ..., w_{(n)}$ denote the ordered importance weights, and $w_{(k)}$ be the weight at the $x$-th percentile. The truncated weights $\tilde{w}_i$ are computed as: $\tilde{w}i = \min(w_i, w{(k)})$.
The normalized weights are then given by:$\bar{w}_i = \frac{\tilde{w}i}{\sum{\tilde{w}_j} }$.

It should be noted that excessive weight truncation may introduce bias into the original importance sampling estimator, as the unbiasedness may no longer hold after truncation. Therefore, the choice of the truncation percentile $x$ requires careful consideration to balance between variance reduction and bias introduction. We investigate the relationship between the percentile of truncated weights and the final performance in our framework, as presented in the \autoref{tab:truncate}. From the table, it can be observed that extreme values tend to emerge when the truncation weight percentile is either too large or too small. We select $x = 0.9$ as the truncation percentile.

\paragraph{The size of the prompt set.}
Since our framework approximates the expectation by summing the sample scores, the number of samples may also become a factor affecting the performance of our framework. An insufficiently large prompt set size could lead to unstable predictions. To verify this hypothesis, we design five different settings for two prompt sets: \textit{full}, 700, 500, 300, and 100, representing using the complete dataset or randomly sampled subsets of 700, 500, 300, and 100 data points respectively. Under these five configurations, we conduct mutual prediction between two datasets to observe how the prediction results vary with different sample sizes.

We present the testing results in Table~\ref{tab:bigcode_evo_scores_70b}. Each cell in the table contains two numbers separated by a forward slash (/). The value on the left side of the slash originates from the BigCode-predicts-Evo dataset, while the value on the right represents the inverse scenario. As demonstrated, the reduction in size for either the target or source prompt set leads to a significant performance degradation. Although our experimental dataset sizes are typically on the order of thousands, we recommend using larger datasets in practical applications. 
 
\answer{4}{We observe that the choice of dimensionality reduction method and the resulting reduced dimension both affect the performance of our framework. Statistical approaches, such as PCA, demonstrate significantly better results compared to neural network-based methods. Additionally, as the reduced dimension becomes lower, the error of our framework increases. In practical applications, truncating the weight percentiles and the number of IWAE samples also influence the performance of our framework.}
\section{Discussion}
\label{sec:discussion}
Our work represents the first exploration of a benchmark-free approach in the domain of code tasks. In this section, we identify several promising avenues for future research:

\textbf{Anti Data contamination:} Data contamination\cite{fu2024does} occurs when open-source test suites and their accompanying reference solutions are incorporated into the LLM's training data, causing validity of benchmark degrade over time\cite{li2024task}. This issue is not unique to code benchmarks\cite{zhou2025lessleak} but in many other areas\cite{gupta2024changing,liu2023tinygsm};This phenomenon can lead to artificially inflated performance on the test set, with another study indicating that LLM scores on contaminated samples can be up to five times higher than on uncontaminated samples\cite{zhou2025lessleak}. Since our framework does not need test suite and reference answer, we can significantly reduce the risk of data contamination.

 \textbf{Exploration of Cross-Scenario and Cross-Language Code Testing:} Although our evaluation utilized two distinct benchmarks, they correspond to largely consistent coding scenarios. Testing across significantly divergent scenarios and languages remains unexplored in this work. We posit this as a highly valuable future direction. Another critical aspect involves scenarios where prompt distributions exhibit minimal overlap, causing the computed weights might disproportionately favor a minority of samples. Investigating methods to mitigate this issue is crucial for extending the applicability of our framework to broader contexts.

\textbf{Leveraging Closed-Source Benchmarks for Anti-Cheating Verification:} As discussed, existing benchmarks are often compromised by data contamination. However, employing closed-source benchmarks within our framework—where the source benchmark remains inaccessible for learning—could enable the reverse testing of potential cheating by existing LLMs on various benchmarks. This is feasible because, when the source benchmark is free from data contamination, our framework can correspondingly compute the actual performance of LLMs under uncontaminated conditions.

\section{Conclusion}
\label{sec:conclusion}
This paper introduces an innovative approach \approach{} that combines \emph{Importance Sampling} with \emph{Importance Weighted Autoencoders (IWAE)} to evaluate the performance of LLM on code generation tasks without requiring human annotators or LLMs as judges. This approach provides new perspectives for automatic LLM evaluation and can reduce the cost of high-quality benchmark construction.

Future work will focus on further optimizing the IWAE architecture and exploring better approaches to handle out-of-distribution samples. Meanwhile, considering that real-world code generation tasks often involve complex interactive processes, how to evaluate LLM performance in dynamic environments in real-time represents another valuable direction for in-depth investigation.

In conclusion, this study provides a novel pathway for the automated and efficient evaluation of LLMs in code generation, contributing to advancements in the field of artificial intelligence. With technological progress, we anticipate seeing more innovative applications emerge based on this framework.

\section{Acknowledgement}
This research is supported by the Ministry of Education,
Singapore under its Academic Research Fund Tier 3 (Award
ID: MOET32020-0004). Any opinions, findings and conclusions or recommendations expressed in this material are those
of the author(s) and do not reflect the views of the Ministry
of Education, Singapore.

\bibliographystyle{ACM-Reference-Format}
\bibliography{sample-base}

\end{document}